\documentclass[letterpaper, 10 pt, conference]{ieeeconf}

\IEEEoverridecommandlockouts                              %

\title{\LARGE \bf
	The Role of Symmetry in Constructing Geometric \\ Flat Outputs for Free-Flying Robotic Systems
}

\author{Jake Welde, Matthew D. Kvalheim, and Vijay Kumar
\thanks{
	J. Welde and V. Kumar are with the GRASP Laboratory at the University of Pennsylvania, while M. D. Kvalheim is with the Department of Mathematics at the University of Michigan. emails: \texttt{\{jwelde,kumar\}@seas.upenn.edu}, \texttt{kvalheim@umich.edu}.  
	We sincerely thank Prof. Muruhan Rathinam for reviewing several drafts of the manuscript and providing valuable feedback, including the important observation that differential flatness is already a geometric property i.e. all flat outputs are in a sense geometric; we use the term rather to highlight a connection with geometric mechanics that has not previously been understood. We also thank Prof. Richard Murray and Prof. Anthony Bloch for helpful discussions throughout the development of this work. 	We gratefully acknowledge the support of Qualcomm Research, NSF Grant CCR-2112665, and the NSF Graduate Research Fellowship Program.
}
}
\date{\today}

\usepackage{titlesec}
\usepackage{mathtools}
\usepackage{xcolor}
\usepackage{amsthm}
\usepackage{amsmath}
\usepackage{amssymb}
\usepackage{thmtools}
\usepackage{graphicx,subcaption}
\usepackage{svg}
\usepackage{tikz}
\usepackage{graphicx}
\usepackage{enumerate}%
\usepackage{bbm}
\usepackage{tablefootnote}
\usepackage{dblfloatfix}
\usepackage[hidelinks]{hyperref}

\usetikzlibrary{arrows, arrows.meta, 
	backgrounds,
	calc,
	decorations.pathmorphing,
	patterns, positioning, 
	quotes,
	shapes,
	tikzmark
}

\usepackage{pgfplots}
\pgfplotsset{compat = newest}

\usetikzlibrary{positioning}
\usetikzlibrary{arrows,decorations.markings}

\usepackage{tikzscale}

\usepackage{quiver}
\tikzcdset{every label/.append style = {font = \normalsize}}

\declaretheorem{theorem}

\declaretheorem{proposition}

\declaretheoremstyle{normaltext}
\declaretheorem[style=normaltext]{definition}
\declaretheorem[style=normaltext]{remark}

\declaretheorem[style=normaltext]{fact}
\declaretheorem[style=normaltext]{example}

\renewcommand\thmcontinues[1]{\textit{continued}}

\newcommand{\metric}[2]{\langle\langle{\hspace{1pt}}#1{\hspace{1pt},\hspace{1pt}}#2{\hspace{1pt}}\rangle\rangle}

\newcommand{\pairing}[2]{\langle#1{\hspace{1pt};\hspace{1pt}}#2\rangle}
\newcommand{\bracket}[2]{[#1{\hspace{1pt},\hspace{1pt}}#2]}

\newcommand{\curve}[0]{{q}}

\newcommand{\cdac}[0]{{\nabla_{\dot{\curve}} \hspace{1pt}}}

\newcommand{\vf}[1]{\Gamma(T{#1})}
\newcommand{\region}{U}
\newcommand{\unactuated}{X}

\DeclareMathOperator{\spn}{span}

\DeclareMathOperator{\rank}{rank} 
\DeclareMathOperator{\grad}{grad}

\DeclareMathOperator{\rot}{rot}

\newcommand{\flatmodifier}[0]{geometric }

\renewenvironment{quote}{%
	\list{}{%
		\leftmargin0.5cm   %
		\rightmargin\leftmargin
	}
	\item\relax
	\vspace{6pt}
	\begin{center}
		\itshape	
	}
	{	
	\end{center}
	\vspace{6pt}
	\endlist
}

\usepackage{cite}

\usepackage{balance}

\begin{document}
	
	\begingroup
	\let\cleardoublepage\relax
	\let\clearpage\relax
	\maketitle
	
	\begin{abstract}

Mechanical systems naturally evolve on principal bundles describing their inherent symmetries. The ensuing factorization of the configuration manifold into a \textit{symmetry group} and an internal \textit{shape space} has provided deep insights into the locomotion of many robotic and biological systems. On the other hand, the property of \textit{differential flatness} has enabled efficient, effective planning and control algorithms for various robotic systems. Yet, a practical means of finding a \textit{flat output} for an arbitrary robotic system remains an open question. In this work, we demonstrate surprising new connections between these two domains, for the first time employing symmetry directly to construct a flat output. We provide sufficient conditions for the existence of a \textit{trivialization} of the bundle in which the group variables themselves are a flat output. We call this a \textit{\flatmodifier flat output}, since it is \textit{equivariant} (i.e. it preserves the symmetry) and often \textit{global} or \textit{almost global}, properties not typically enjoyed by other flat outputs.
In such a trivialization, the motion planning problem is easily solved, since a given trajectory for the group variables will fully determine the trajectory for the shape variables that exactly achieves this motion. We provide a partial catalog of robotic systems with \flatmodifier flat outputs and worked examples for the planar rocket, planar aerial manipulator, and quadrotor.

	\end{abstract}

	\endgroup

\section{Introduction}

\textit{Differential flatness} is a strong property of certain control systems that has been exploited for effective planning and control of highly dynamic maneuvers for underactuated robots\cite{Mellinger2011,Tang2018,Thomas2017a}. 
Such approaches are also ideally suited to realtime operation on systems subject to \textit{size, weight, and power} (SWaP) constraints due to their computational efficiency.
For example, a  flatness-based controller for a racing drone executing extremely aggressive trajectories at the edges of the flight envelope showed only a small reduction in tracking performance as compared to a nonlinear model predictive controller; nonetheless, it consumed about 100 times less computational power \cite{Sun2022}. However, finding the requisite \textit{flat output} for a given system (or even determining whether one exists) is a challenging task usually achieved by manual trial and error, since the necessary and sufficient conditions for flatness of underdetermined systems \cite{Levine2011} are broadly speaking too general to be tractably applied to multibody robotic systems, whose equations of motion grow rapidly in complexity with the number of bodies.

Nevertheless, numerous mechanical systems have been found to be differentially flat \cite{Murray1995}, and previous work has employed the Riemannian structure of the equations of motion of certain classes of mechanical systems to obtain tractable conditions for flatness. 
For unconstrained systems with no more than one unactuated degree of freedom, \cite{Rathinam1998} gave a constructive necessary and sufficient condition to obtain a \textit{configuration flat output} (a function of the configuration alone and not the velocities, inputs, or higher derivatives). These results were extended in \cite{Sato2012} to systems with more unactuated degrees of freedom, but a candidate output was assumed to be given. Furthermore, the flat outputs obtained in both these methods rely on local coordinates, and thus they are only valid locally around a nominal operating point.

Mechanical systems often exhibit \textit{symmetry}, which induces a \textit{principal bundle} structure on the configuration manifold and (by Noether's Theorem) the conservation of momentum \cite{Bloch2005}. The role of symmetry in locomotion has been studied in detail for systems whose evolution is governed by a \textit{principal connection} \cite{Ostrowski}. It has long been conjectured that symmetry also plays a role in differential flatness \cite{Murray1995,Murray1997}, however a clear link between the two concepts has yet to be established.  In \cite{Rathinam1998}, it was shown that any flat output of a very special class of symmetric systems must be equivariant; however, symmetry was not leveraged to construct the outputs, nor was equivariance guaranteed for more general systems with symmetry.

\begin{table*}[t]
	\centering
	\def\arraystretch{1.5}%
	\vspace{2pt}
	\caption{A partial catalog of unconstrained mechanical systems known to possess \flatmodifier flat outputs.}
	\begin{tabular}{|c|c|c|c|c|c|}
		\hline
		
		\textbf{System} & \textbf{Configuration Manifold $Q$}  & \textbf{Symmetry Group $G$} & \textbf{Shape Space $S$} & \textbf{Bundle} & \textbf{Reference} \\ \hline

		Planar Rocket (\textit{Fig. \ref{fig:planar_rocket}})   & $SE(2)$ & $\mathbb{R}^2$ & $\mathbb{S}^1$ & Trivial &  \cite{Rathinam1998} \\ \hline

		Planar Aerial Manipulator (\textit{Fig. \ref{fig:planar_aerial_manipulator}}) & $SE(2) \times \mathbb{T}^1$  & $SE(2)$ & $\mathbb{T}^1$ & Trivial & \cite{Thomas2014}  \\ \hline
		
		Aerial Manipulator in 3D & $SE(3) \times \mathbb{T}^2$  & $SE(3)$ & $\mathbb{T}^2$ & Trivial & \cite{Welde2021}  \\ \hline

		Quadrotor  (\textit{Fig. \ref{fig:quadrotor}}) 		& $SE(3) $  & $\mathbb{R}^3 \times \mathbb{S}^1$ 
		& $\mathbb{S}^2$ & Nontrivial & \cite{Watterson2017}
		\\ \hline
		
		Quadrotor with Slung Load 
		& $SE(3) \times \mathbb{S}^2$  & $\mathbb{R}^3 \times \mathbb{S}^1$ & $\mathbb{S}^2 \times \mathbb{S}^2$ & Nontrivial & \cite{Sreenath2013} \\ \hline
		
		Chain of $n$ Spring-Mass Systems &  $\mathbb{R}^n$ & $\mathbb{R}$ & $\mathbb{R}^{n-1}$ &  Trivial & \cite{Blajer2013} \\ \hline
		
	\end{tabular}
	\label{catalog_of_\flatmodifier_flat_systems}
	\vspace{-10pt}
\end{table*}

However, many known flat outputs of mechanical systems (e.g. those cataloged in Table \ref{catalog_of_\flatmodifier_flat_systems})
 can be interpreted as the group variables of a \textit{trivialization} of a principal bundle describing a symmetry of the system. In a similar vein, \cite{Dear2015} examined several locomotion systems with group and shape spaces of equal dimension, showing \textit{partial differential flatness} (a more general yet significantly weaker property) with respect to the group variables of a trivialization, although none were found to be fully differentially flat. With these observations in mind, we raise the following question:
\begin{quote}
	When does a mechanical system with symmetry admit a trivialization of its principal bundle 
	in which the group variables are a flat output of the system?
\end{quote}
Such \textit{\flatmodifier flat outputs} have powerful properties:
\begin{enumerate}[I.]
	\item 	\label{configuration_flatness}
	They are \textit{configuration} flat outputs \cite{Rathinam1998}, which makes it convenient to write constraints on positions, velocities and inputs in terms of flat output derivatives \cite{Thomas2017a}.
	\item \label{equivariance}
	They are \textit{equivariant}, meaning that they preserve the fundamental symmetry found in the physical system, thereby avoiding the introduction of artificial bias.
	\item \label{globality} They are often \textit{global or almost global}, meaning the flat output map is well-defined over all or almost all of the configuration manifold, encompassing a more complete portion of the system's performance envelope.
\end{enumerate}

In what follows, we describe the dynamics of unconstrained mechanical systems with symmetry using the geometry of Riemannian manifolds and principal bundles. We give a simple but important necessary condition and a practical sufficient condition for the existence of a \flatmodifier flat output. The latter amounts to a mild regularity criterion and the existence of a \textit{section} of the bundle that is orthogonal to a certain computable distribution, from which we immediately construct the geometric flat output. We discuss the implications of a nontrivial bundle structure in regards to prohibiting the existence of truly global \flatmodifier flat outputs. The approach is applied to several illustrative example systems to demonstrate the key concepts. 
The results in this paper have broad applications in the control of an important class of multibody robotic systems, including many vehicles and manipulators operating in air or space environments.

\section{Principal Bundle Geometry}

To start, we give a very brief review of principal bundle geometry. We direct the less familiar reader to 
\cite[Chapter 2]{Bloch2005} and \cite[Chapters 3 and 4]{BulloAndLewis} for a far more complete introduction to principal bundles, symmetry, and Riemannian geometry applied to the control of mechanical systems.
\begin{definition}
	A \textit{principal bundle} is a manifold $Q$ with:
	\begin{enumerate}[i.]
		\item ${\Phi_g : Q \rightarrow Q}$ (sometimes written ${g \cdot q}$), a smooth, free, and proper action of a Lie group $G$ on $Q$, and
		\item ${\pi : Q \rightarrow S}$, the smooth projection to the manifold of equivalence classes induced by the action.
	\end{enumerate}
	We call $S$ the \textit{shape space} and $G$ the \textit{symmetry group}.
\end{definition}

\begin{definition}
	\normalfont
	A \textit{section} of a bundle is a smooth map ${\sigma : \region \subseteq S \rightarrow Q}$ such that ${\pi \circ \sigma(s) = s}$ for all ${s \in \region}$. 
\end{definition}

\begin{definition}
	A \textit{trivialization} is a diffeomorphism 
	\begin{equation}
		\begin{gathered}
			\psi : \pi^{-1}(\region) \rightarrow \region \times G \\
			q \mapsto \big(\pi(q), \varphi(q)\big)
		\end{gathered}
	\end{equation}
	such that $\varphi$ is equivariant i.e. 
	${\varphi (g \cdot q) = g \cdot \varphi(q)}$. 
	A bundle is \textit{trivial} when there exists a \textit{global} trivialization (i.e. a trivialization for which ${\region = S}$), otherwise it is \textit{nontrivial}.
\end{definition}

\begin{fact}
	\label{identification_of_sections_and_trivializations}
	The sections 
	and trivializations of a principal bundle
	are in one-to-one correspondence, identified by the relation
	${\psi^{-1}(s,g) = \Phi_g \circ \sigma(s)}$.
	In the previous, $\psi$ is called the \textit{canonical trivialization} corresponding to $\sigma$, and likewise $\sigma$ is called the \textit{canonical section} corresponding to $\psi$.
\end{fact}

	\begin{definition}
	\normalfont	The \textit{vertical subbundle} ${VQ \subset TQ}$ is given by 
	$VQ = \ker T\pi$. \normalfont	A \textit{principal connection} is the assignment of a \textit{horizontal subbundle} $HQ \subset TQ$ that is:
	\begin{enumerate}[i.]
		\item complementary to $VQ$ (i.e. ${TQ = VQ \oplus HQ}$), 
		\item equivariant (i.e. $T_q \Phi_g \, H_qQ = H_{g\cdot q} Q$), and
		\item assigned smoothly over $Q$.
	\end{enumerate}
\end{definition}

\begin{definition}
	\label{canonical_flat_connection}
	Given a trivialization ${q \mapsto \big(\pi(q),\varphi(q)\big)}$,
	the \textit{canonical flat connection} \cite[Chapter 2, Section 9]{KobayashiAndNomizu1} is the principal connection satisfying ${HQ = \ker T\varphi}$. 
	It is readily seen that the canonical flat connection is tangent to the image of the canonical section as a submanifold of $Q$.
\end{definition}
\begin{definition}
	Given a basis $e_a$ for $\mathfrak{g}$ (the Lie algebra of $G$) and coordinates $s^\alpha$ for $S$, the \textit{standard basis} for $TQ$ associated with a particular principal connection $HQ$
	is the union of the bases of vector fields%
	\footnote{Throughout, 
		we use index notation and the Einstein summation convention. For clarity, we use indices 
		$a,b,c$ for $G$  and $\alpha, \beta, \gamma$ for $S$.}
	\begin{equation}
		V_a = (e_a)_Q, \ \ H_\alpha = \left(\tfrac{\partial}{\partial s^\alpha}\right)^{HQ}
	\end{equation}
	spanning $VQ$ and $HQ$.
	Here, ${\xi_Q \in \Gamma(TQ)}$ is the \textit{infinitesimal generator} associated with any ${\xi \in \mathfrak{g}}$ given by $
		{q \mapsto 
			\left.
			\frac{d}{dt}\big(
			\Phi_{\exp t \xi} \ q 
			\big)
			\right|_{t = 0}}$, while
	${X^{HQ} \in \Gamma(HQ)}$ is the \textit{horizontal lift} of any vector field ${X \in \Gamma(TS)}$ i.e. the unique horizontal vector field satisfying $T\pi \circ X^{HQ} = X$.
\end{definition}
	Note that $\Gamma(TM)$ denotes the set of vector fields on $M$.

\section{Unconstrained Mechanical Systems}

\begin{definition}
	\normalfont		
	An \textit{unconstrained mechanical system on a principal bundle} (briefly, a \textit{system}) consists of:
	\begin{enumerate}[i.]
		\item a \textit{Lagrangian} ${L : TQ \rightarrow \mathbb{R}}$ in the form
		\begin{equation}
			\begin{gathered}
				v_q \mapsto K(v_q) - P \circ \pi_Q(v_q),
			\end{gathered}
		\end{equation}
		where ${K : v_q \mapsto \frac{1}{2}\metric{v_q}{v_q}}$ is the \textit{kinetic energy} described using a Riemannian metric, ${P: Q \rightarrow \mathbb{R}}$ is the \textit{potential energy}, and ${\pi_Q : TQ \rightarrow Q}$ is the standard tangent bundle projection,
		\item a \textit{control codistribution} ${F \subseteq T^*Q}$, containing the controlled external forces that can be applied, and
		\item a (possibly nontrivial) \textit{principal bundle structure} ${\pi : Q \rightarrow S}$ induced by a $G$-action ${\Phi_g : Q \rightarrow Q}$.
	\end{enumerate}
\end{definition}
Note that the term \textit{unconstrained} refers to the absence of nonholonomic constraints; clearly, a system evolving on a manifold is in a certain sense subject to holonomic constraints describing the embedding of the manifold in some ambient Euclidean space.

\begin{definition}
	A system is \textit{symmetric} if $F$ is equivariant with respect to the action,
	i.e. $T\Phi_g^* \, F_q = F_{g \cdot q}$,
	and $L$ is invariant with respect to the lifted action,
 	 i.e. $L \circ T\Phi_g (v_q) = L(v_q)$.
	 If $K$ is invariant and $F$ is equivariant but $P$ is not invariant, the system is said to exhibit \textit{broken symmetry} \cite{Contreras2022}.
\end{definition}

The broken symmetry case is often seen in practice for aerial robots operating under the effect of gravity.

\begin{figure*}[t]
	\begin{minipage}{.3\textwidth}
	\centering
	\includegraphics{rocket.tikz} 
	\captionof{figure}{Planar Rocket}
	\label{fig:planar_rocket}
\end{minipage}%
	\begin{minipage}{.4\textwidth}
		\centering
		\includegraphics{aerial_manipulator.tikz}
		\captionof{figure}{Planar Aerial Manipulator}
		\label{fig:planar_aerial_manipulator}
	\end{minipage}%
	\begin{minipage}{.3\textwidth}
		\centering
		\includegraphics{quadrotor.tikz}
		\captionof{figure}{Quadrotor}
		\label{fig:quadrotor}
	\end{minipage}
	\vspace{-17pt}
\end{figure*}

\begin{fact}[Forced Geodesic Equation]
	Unconstrained mechanical systems evolve along curves ${\curve : \left[0,1\right] \rightarrow Q}$ satisfying the system of differential equations given by
	\begin{equation}
		\cdac \dot{\curve} + \grad P = f^\sharp,
		\label{eq:forced_geodesic_equation}
	\end{equation}
	where ${\nabla : \vf{Q} \times \vf{Q} \rightarrow \vf{Q}}$ is the \textit{Riemannian} (or \textit{Levi-Civita}) \textit{connection} arising from the kinetic energy, $\cdac$ is the \textit{covariant derivative} induced along the curve $\curve$, and ${f: \left[0,1\right] \rightarrow F}$ are the applied external control forces.\footnote{\, The map ${\cdot^\sharp : T^*Q \rightarrow TQ}$ is the \textit{musical isomorphism} satisfying ${\metric{f^\sharp}{v} = \pairing{f}{v}}$, whose inverse is denoted ${\cdot^\flat : TQ \rightarrow T^*Q}$.}
\end{fact}

The previous equation is the Riemannian equivalent of the Euler-Lagrange equations more often seen in robotics.

\begin{definition}
	\normalfont
	The \textit{unactuated subbundle}, denoted $UQ$, is the coannihilator of the control codistribution, 
	given by
	\begin{equation}
		U_q Q = \big\{ u_q \in T_q Q  :   \pairing{f_q}{u_q} = 0 \ \forall \   f_q \in F_q  \big\},
	\end{equation}
	where $\pairing{\cdot}{\cdot}$ is the natural pairing of vectors and covectors. 
	A system is said to be \textit{fully actuated} when $UQ$ is empty and \textit{underactuated} with degree $u = \rank UQ$ otherwise.
\end{definition}

\begin{definition}
	A curve ${\curve}$ is \textit{dynamically feasible} if there exists a covector field ${f}$ along $\curve$ 
	satisfying 
	\eqref{eq:forced_geodesic_equation}.
\end{definition}

\begin{proposition}[Implicit Dynamics]
	A given curve ${{\curve} : \left[0,1\right] \rightarrow Q}$ is dynamically feasible if and only if 
	\begin{equation}
		\label{eq:implicit_dynamics}
		\metric{\unactuated}{\cdac \dot{\curve} + \grad P} = 0  \textrm{\, for all \,} \unactuated \in \Gamma(UQ).
	\end{equation}
\end{proposition}
\begin{proof}
	This amounts to a projection of the forced geodesic equation onto the unactuated subbundle, which eliminates the external control forces since $\metric{X}{f^\sharp} = \pairing{f}{X} = 0$. 
\end{proof}

\begin{proposition}[Local Implicit Dynamics]
	\label{implicit_dynamics_in_components}
	For any principal connection $HQ$ and any trivialization ${\psi : \pi^{-1}(\region) \rightarrow \region \times G}$%
	 , the implicit dynamics
	can be expressed 
	locally
	as the zero level set of an $\mathbb{R}^{u}$-valued map
	\begin{equation}
			\label{vector_valued_implicit_dynamics}
			E_i(s^\alpha, \dot{s}^\alpha,\ddot{s}^\alpha; g,\xi^a,\dot{\xi}^a) = 0,
	\end{equation}
	where for any given curve ${q : \left[0,1\right] \rightarrow Q}$, the functions ${\xi^a, \dot{s}^\alpha : \left[0,1\right] \rightarrow \mathbb{R}}$ satisfy ${\dot{q} = \xi^a V_a + \dot{s}^\alpha  H_\alpha}$, i.e. they express the curve's velocity in the standard basis.\footnote{While $\dot{s} = \dot{s}^\alpha \tfrac{\partial}{\partial s^\alpha}$ is always the \textit{shape velocity}, only when $HQ$ is the canonical flat connection is $\xi = \xi^a e_a$ the \textit{spatial velocity} associated with the group variables i.e. $\xi = \dot{g} \cdot g^{-1}$. In Lagrangian reduction, where $HQ$ is the \textit{mechanical connection}, $\xi$ is called the \textit{locked velocity} \cite{BKMM}.}
\end{proposition}
\begin{proof}
Via the basic properties of affine connections, we may express the geometric acceleration of any such curve as \cite{Bullo1996}
\begin{equation}
	\label{\flatmodifier_acceleration}
	\begin{aligned}
		\cdac \dot{\curve} = 
		\dot{\xi}^a &V_a + 
		\ddot{s}^\alpha  H_\alpha + 
		\xi^a  \xi^b  \nabla_{V_a} {V_b} \ + 
		\\
		&\ 2 \hspace{1pt} \xi^a  \dot{s}^\beta \nabla_{V_a} {H_\beta} +
		\dot{s}^\alpha  \dot{s}^\beta  \nabla_{H_\alpha} {H_\beta} 
		,
	\end{aligned}
\end{equation}
where we can collect the cross terms due to the equivariance of horizontal lifts i.e. 
$\nabla_{V_a}{H_\beta} - \nabla_{H_\beta}{V_a} = 
{\bracket{V_a}{H_\beta} = 0}$. 
Furthermore, a basis $X_i$ for $UQ$ exists over any sufficiently local region of $\pi^{-1}(U)$, and
due to the bilinearity of the metric, \eqref{eq:implicit_dynamics} holds if and only if it holds for each $X_i$ in the basis. Therefore, substituting \eqref{\flatmodifier_acceleration} into \eqref{eq:implicit_dynamics} for each $X_i$ and expressing $q$ in terms of $s$ and $g$ using the trivialization gives us each component of our local vector-valued constraint.
\end{proof}

The following distribution, a generalization of one defined in \cite{Rathinam1998}, will feature in our sufficient condition for flatness, and can be easily computed directly from the system model.
\begin{definition}
	The \textit{underactuation distribution} $\Delta \subseteq TQ$ is 
	\begin{equation}
		\Delta = \spn \big\{X, \nabla_Y X : X \in \Gamma(UQ), \ Y \in \Gamma(TQ) \big\}.
	\end{equation}
\end{definition}

	\begin{proposition}
		\label{equivariance_of_delta}
		For any system with (broken) symmetry, $\Delta$ is equivariant.
	\end{proposition}
	\begin{proof}
		Following the same line of reasoning as \cite{Rathinam1998}, the claim follows from the equivariance of $F$ and thus of $UQ$,
		and the fact that for the Riemannian connection of an invariant metric, the covariant derivative of an equivariant vector field along another equivariant vector field is  equivariant. 
	\end{proof}

\subsection{Example Systems}

We now introduce three systems (see Fig. \ref{fig:planar_rocket}-\ref{fig:quadrotor}) to serve as running examples for the remainder of the paper.

\begin{example}[name=Planar Rocket,label=example:planar_rocket] 
	Also known as the ducted fan, this classic example of a flat system has configuration manifold ${Q = SE(2)}$, to which we assign coordinates ${q = (x_1,x_2,\theta)}$ (and the corresponding basis of coordinate vector fields) corresponding to the center of mass position and the angle of body rotation. 
	The system is given by\footnote{In examples, (co)vector fields are described by column vectors whose components are their coefficients in a chosen basis of (co)vector fields, while (co)distributions are described by matrices whose columns are (co)vector fields that together span the (co)distribution, and the metric is described by the components of the inertia tensor in the basis of vector fields.} 
	\begin{equation}
			\setlength\arraycolsep{2pt}
		M = \begin{bmatrix}
			m & 0 & 0 \\
			0 & m & 0 \\
			0 & 0 & J
		\end{bmatrix}, \,
		P = m\,g\,x_2, \,
		F = \begin{bmatrix}
			\cos \theta & -\sin \theta \\
			\sin \theta &  \cos \theta \\
			r &        0
		\end{bmatrix}
	\end{equation}
	which exhibit broken symmetry with respect to the action of ${(g_1,g_2) \in \mathbb{R}^2}$ given by 
	\begin{equation}
		\Phi_g : (x_1,x_2,\theta) \mapsto (x_1+g_1,x_2+g_2,\theta).
	\end{equation}
	This action induces a principal bundle structure over $\mathbb{S}^1$, where the projection map is given by ${\pi : SE(2) \rightarrow \mathbb{S}^1 , \, q \mapsto \theta}$.
	The unactuated subbundle and underactuation distribution can be computed as
	\begin{equation}
			UQ = \begin{bmatrix}
			-r \cos \theta \\
			-r \sin \theta \\
			1
		\end{bmatrix}, \ \Delta = 
		\begin{bmatrix}
			-r \cos \theta & \hphantom{-} r \sin \theta \\
			-r \sin \theta & -r \cos \theta \\
			1 &          0
		\end{bmatrix}.
	\end{equation}
\end{example}
\begin{example}[name=Planar Aerial Manipulator,label=example:planar_aerial_manipulator]

	The configuration manifold is ${Q = SE(2) \times \mathbb{T}^1}$, to which we assign coordinates ${q = (x_1,x_2,\theta,\phi)}$ such that ${(x_1,x_2,\theta)}$ describe the end effector pose in $SE(2)$ while ${\phi \in \mathbb{T}^1}$ is the joint angle. The system can be shown to exhibit  broken symmetry\footnote{The explicit form of the metric tensor is somewhat lengthy even for this two-body system, so we list only the essential details for brevity.} with respect to displacements in the plane, which can be described as the action of ${(g_1,g_2,g_3) \in SE(2)}$ as 
	\begin{equation}
		\Phi_g : 
		\begin{pmatrix}
			x_1 \\ 
			x_2 \\ 
			\theta \\ 
			\phi 
		\end{pmatrix}
		\mapsto
		\begin{pmatrix}
			x_1 \cos g_3 - x_2 \sin g_3 + g_1 \\ 
			x_1 \sin g_3 + x_2 \cos g_3 + g_2 \\ 
			\theta + g_3\\ 
			\phi 
		\end{pmatrix},
	\end{equation}
	inducing a principal bundle structure over $\mathbb{T}^1$, where the projection map is given by ${\pi : SE(3) \times \mathbb{T}^1 \rightarrow \mathbb{T}^1, \, q \mapsto \phi}$. The unactuated subbundle and the underactuation distribution can be computed as
	\begin{equation}
			UQ = 
		\begin{bmatrix} 
			\setlength\arraycolsep{3pt}
			-\sin(\phi + \theta) \\
			\hphantom{-}\cos(\phi + \theta) \\
			0 \\
			0
		\end{bmatrix}
		, \ \
		\Delta = 
		\begin{bmatrix} 
			\setlength\arraycolsep{3pt}
			1 & 0 \\
			0 & 1 \\ 
			0 & 0 \\ 
			0 & 0    
		\end{bmatrix}.
	\end{equation}
\end{example}
\begin{example}[name=Quadrotor,label=example:quadrotor]

	\label{quadrotor_example}

	The configuration manifold is ${Q = SE(3)}$, which can be parametrized using homogeneous coordinates as
	\begin{equation}
	\setlength\arraycolsep{3pt}
q = \begin{pmatrix}
	R & x \\
	0_{1\times 3} & 1
\end{pmatrix}
	\end{equation}
	where ${R \in SO(3)}$ is the rotation from the body frame to the world frame and ${x \in \mathbb{R}^3}$ is the position of the center of mass.
	We use the usual basis of vector fields for $TSE(3)$, corresponding to the components of the linear and angular velocities along body-fixed axes. The system is given by\footnote{Assuming the rotational inertia is symmetric about the thrust axis gives the system broken symmetry for the Abelian group action given, as opposed to considering a non-Abelian subgroup of $SE(3)$. However, the flat outputs obtained are valid even without this assumption due to Remark \ref{relax_symmetry}.}
	\begin{equation}
		\begin{gathered}
				\setlength\arraycolsep{3pt}
			M = 
			\textrm{diag}(
			\begin{bmatrix}
				m & m & m & J_{xx} & J_{xx} & J_{zz}
			\end{bmatrix}
			), 
			\\
			P = m\,g\,x_3, \ \ 
			F = \begin{bmatrix}
				e_3 & 0_{3 \times 3} \\
				0_{3 \times 1} & \mathrm{I}_{3\times3}
			\end{bmatrix}
		\end{gathered}
	\end{equation}
	which exhibit broken symmetry with respect to the action of 
	$(g_{123},g_4) \in \mathbb{R}^3 \times \mathbb{S}^1$ 
	given by 
	\begin{equation}
			\setlength\arraycolsep{3pt}
		\Phi_g : 
		\begin{pmatrix}
			{R} & {x} \\
			0_{1\times 3} & 1
		\end{pmatrix} 
		\mapsto
		\begin{pmatrix}
			R \cdot \rot_{e_3}(g_4) & x + {g}_{123} \\
			0_{1\times 3} & 1
		\end{pmatrix}
		.
	\end{equation}
	comprised of a translation in all three world-fixed axes and a rotation around the body-fixed thrust axis. 
	This induces a nontrivial bundle over $\mathbb{S}^2$, where the projection map is ${\pi : SE(3) \rightarrow \mathbb{S}^2 , \, q \mapsto R \cdot e_3}$. The unactuated subbundle and underactuation distribution are given by 
	\begin{equation}
			\setlength\arraycolsep{3pt}
		UQ = \begin{bmatrix}
			e_1 & e_2 \\
			0_{3 \times 1} & 0_{3 \times 1}
		\end{bmatrix} 
		, \ \ 
		{\Delta} = 
		\begin{bmatrix}
			\mathrm{I}_{3\times3} \\
			0_{3\times3}
		\end{bmatrix}. 
	\end{equation}
\end{example}

\section{Geometric Flat Outputs}

\begin{definition}
	An \textit{output} of a system of differential equations ${f(x,\dot{x}) = 0}$ is a map
			$(x,\dot{x},\ldots,x^{(p)}) \mapsto y$ i.e.
the output value $y$ depends on finitely many derivatives of $x$.
	An output is \textit{differentially flat} if generic\footnote{Throughout, by \textit{generic} we mean \textit{belonging to an open dense subset (in the $C^\infty$ topology)} \cite{Hirsch2012}. This corresponds roughly with the colloquial meaning of ``virtually all'' members of the set.} solutions ${x : \left[0,1\right] \rightarrow \mathcal{X}}$ of the system are in locally unique correspondence with generic smooth curves ${y : \left[0,1\right] \rightarrow \mathcal{Y}}$ in the output space, i.e. locally there also exists a map
		$(y,\dot{y},\ldots,y^{(r)}) \mapsto x$
	expressing the original trajectory of the system in terms of finitely many derivatives of the output $y$.
	A system is \textit{(differentially) flat} if it has a flat output.
\end{definition}

Differential flatness thus establishes an equivalence between a physical system and system in the ``flat space'' that need not obey any constraints besides sufficient smoothness. For a more detailed presentation of flatness in terms of ``endogenous transformations'', we direct the reader to \cite{Martin2003}.

\begin{definition}
	A map ${y : \pi^{-1}(\region) \subseteq Q \rightarrow G}$ 
	is a \textit{\flatmodifier flat output} of a mechanical system on a principal bundle if:
	\begin{enumerate}[i.]
		\item $y$ is a flat output of the mechanical system.
		\item The map $\psi : q \mapsto \big(\pi(q),y(q) \big)$ is a trivialization of the principal bundle.
	\end{enumerate}
\end{definition}

\subsection{Necessary Condition}

\begin{theorem}
${\dim G = \rank F}$ is a necessary condition for the existence of a \flatmodifier flat output.
	\label{necessary_conditions}
\end{theorem}
\begin{proof}
	A flat output has the same dimension as the number of equations by which the governing equations are underdetermined \cite{Levine2011}. This is the rank of the control codistribution for unconstrained mechanical systems.
\end{proof}

However, the principal bundle on which a given physical system is described is not unique; for example, the action of a proper subgroup of a larger symmetry group induces a principal bundle structure with a smaller group dimension. 

\subsection{Sufficient Condition}

	Using an argument similar to that of \cite{Rathinam1998} and \cite{Sato2012}, we now examine when the local implicit dynamics simplify further.
	\begin{proposition}
		The local implicit dynamics \eqref{vector_valued_implicit_dynamics} are independent of the shape derivatives if and only if ${HQ \perp {\Delta}}$. 
		\label{shape_inpedendent_implicit_dynamics}
	\end{proposition}
	\begin{proof}
		First, it is clear from the definition of $\Delta$ that the quantities $\metric{H}{X}$ and $\metric{H}{\nabla_Y X}$ vanish identically for all ${H \in \Gamma(HQ)}$, ${X \in \Gamma(UQ)}$ and ${Y \in \Gamma(TQ)}$ if and only if ${HQ \perp {\Delta}}$. Additionally, since $\nabla$ is the Riemannian connection and therefore satisfies the metric compatibility equations, whenever ${HQ \perp \Delta}$ we also have 
		\begin{equation}
			\metric{\nabla_YH}{X} =
			\nabla_Y \metric{H}{X} - \metric{H}{\nabla_YX} = 0.
		\end{equation}  
		Futhermore, it is readily seen that after substituting \eqref{\flatmodifier_acceleration} into \eqref{eq:implicit_dynamics} and expanding using the bilinearity of the metric, all terms in the implicit dynamics depending on the shape derivatives $\dot{s}^\alpha$ or $\ddot{s}^\alpha$ are linear in one of the vanishing terms just discussed. 
		Thus, exactly when $HQ \perp \Delta$, the vector-valued constraint \eqref{vector_valued_implicit_dynamics} takes the form
		\begin{equation}
			\label{simplified_implicit_dynamics}
			\metric{X_i}{
				\dot{\xi}^a V_a + 
				\xi^a  \xi^b  \nabla_{V_a} V_b 
				+ \grad P} = 0. 
		\end{equation}
	which, when combined with the trivialization, yields a constraint in the form
	\begin{equation}
		\label{shape_independent_local_implicit_dynamics}
		E_i(s^\alpha ; g,\xi^a,\dot{\xi}^a) = 0, 
	\end{equation}
concluding the proof.
	\end{proof}

\begin{definition}
	A distribution $D$ is \textit{orthogonal} to a 
	 submanifold ${M \subseteq Q}$ if ${v_q \perp d_q}$ for all ${v_q \in TM, d_q \in D_q}$.
\end{definition}
\begin{definition}
	The local implicit dynamics written using a given trivialization and principal connection $HQ$ are \textit{regular} if 
	${\rank \frac{\partial E_i}{\partial s^\alpha} = \dim S}$ at generic points satisfying \eqref{vector_valued_implicit_dynamics} (or respectively  \eqref{shape_independent_local_implicit_dynamics}, whenever $HQ \perp \Delta$).
\end{definition}
\begin{theorem}
	[Main Result]
	\label{sufficient_condition_for_flatness}	
	Consider an unconstrained mechanical system with (broken) symmetry, for which the dimension of the symmetry group is equal to the rank of the control codistribution. If a section ${\sigma : \region \subseteq S \rightarrow Q}$ of the principal bundle 
	satisfies the following conditions:
\begin{enumerate}[I.]
	\item {\normalfont Orthogonality:} 
	$\Delta$ is orthogonal to the image of $\sigma$ as submanifold of $Q$.
	\label{orthogonality}
	\item 
		\label{regularity}
	{\normalfont Regularity:} The local implicit dynamics are regular when written using the canonical trivialization and canonical flat connection corresponding to $\sigma$.

\end{enumerate}
then the group variables of the canonical trivialization corresponding to $\sigma$ constitute a \flatmodifier flat output.
\end{theorem}

\begin{proof}
In the canonical trivialization, the canonical flat connection $HQ$ is tangent to the image of the canonical section $\sigma$. Thus by the orthogonality condition, $HQ$ is also orthogonal to $\Delta$ over this submanifold in particular. Futhermore, $HQ$ is equivariant by definition and $\Delta$ is equivariant by Proposition \ref{equivariance_of_delta}, and equivariant distributions that are orthogonal at one point along a fiber are orthogonal at every point along that fiber since the metric is invariant. Hence, ${HQ \perp \Delta}
$ everywhere in ${\pi^{-1}(\region) \subseteq Q}$. Thus
 by Proposition \ref{shape_inpedendent_implicit_dynamics},
 the local implicit dynamics for the canonical trivialization and canonical flat connection of $\sigma$ are independent of the shape velocities and accelerations i.e. they take the form \eqref{shape_independent_local_implicit_dynamics}.

The Jacobian $\frac{\partial E_i}{\partial s^\alpha}$ is square since ${\rank UQ = \dim S}$. By the regularity condition, the implicit function theorem applies at generic points satisfying the implicit dynamics, and it follows from continuity that we may locally solve for the shape in terms of $g$, $\xi^a$, and $\dot{\xi}^a$.
It is easily shown that because ${HQ = \ker T\varphi}$ by definition (where $\varphi$ is the $G$-valued part of the canonical trivialization), we have
${\xi^a e_a = \dot{g} \cdot {g^{-1}}}$.
Differentiating to obtain $\dot{\xi}^a$ thus allows us to express the shape in terms of $g$, $\dot{g}$, and $\ddot{g}$, 
and thus reconstruct the configuration $q$ using the trivialization. Differentiating again to obtain the velocity $\dot{q}$ and acceleration $\ddot{q}$ will ultimately yield the inputs $f$ via  the forced geodesic equation. Thus, in the canonical trivialization for $\sigma$, the group variables are a \flatmodifier flat output of the system.
\end{proof}

The attached \href{https://www.youtube.com/watch?v=oMvF86MXTyY}{video} visualizes the proof of Theorem \ref{sufficient_condition_for_flatness}.

\begin{remark}
	\label{mild_nondegeneracy_criterion}

The regularity condition of Theorem \ref{sufficient_condition_for_flatness} is quite mild, and in practice the orthogonality condition is the harder one to satisfy. Singularities are permitted; we only forbid their occurence at \textit{generic} points (since flatness pertains to {generic trajectories}). This is guaranteed if e.g. singularities occur over a closed set of measure zero, which happens often in examples. Note that such singularities do not correspond to points $q \in Q$ but rather tuples $(q, \dot{q}, \ddot{q})$.

\end{remark}

\begin{remark}
	\label{relax_symmetry} The assumption in Theorem \ref{sufficient_condition_for_flatness} that the system exhibits (broken) symmetry can be relaxed to the weaker requirement that $\Delta^\flat \subseteq T^*Q$ is an equivariant codistribution.
\end{remark}

We now apply the main result to the three examples.

\begin{example}[name=Planar Rocket,continues=example:planar_rocket]
	Because we have ${\rank {\Delta}^\perp = \dim S = 1}$ and $G$ is Abelian, by integration it is relatively straightforward to obtain the global section
	\begin{equation}
			\begin{gathered}
			\sigma \,  :\,  \mathbb{S}^1 \rightarrow SE(2) \\
			\theta \mapsto 
			\left(
			\tfrac{J}{m \, r} \sin \theta, \,
			- \tfrac{J}{m \, r} \cos \theta, \,
			\theta
			\right)
		\end{gathered}	
	\end{equation}
	and it can be verified that $\Delta$ is orthogonal to its image using the metric. The vertical and horizontal subbundles are
\begin{equation}
	\setlength\arraycolsep{3pt}
	VQ = \begin{bmatrix}
		1 & 0 \\
		0 & 1 \\
		0 & 0 
	\end{bmatrix}, \ \ 
	HQ = \begin{bmatrix}
		J \cos \theta \\
		J \sin \theta \\
		m\,r
	\end{bmatrix}
\end{equation}
with which the implicit dynamics can also be shown to be regular. Thus by Theorem \ref{sufficient_condition_for_flatness}, the group variables in the canonical trivialization identified with $\sigma$ in Fact \ref{identification_of_sections_and_trivializations}
are a globally valid \flatmodifier flat output, given by
\begin{equation}
	\begin{gathered}
			\renewcommand*{\arraystretch}{1.3}
		y : SE(2) \rightarrow \mathbb{R}^2 
		\\
		q \mapsto 
		(	x_1 - \tfrac{J}{m \, r} \sin \theta ,
		x_2 + \tfrac{J}{m \, r} \cos \theta
		),
	\end{gathered}
\end{equation}
	namely the coordinates of the point $o$ shown in Fig. \ref{fig:planar_rocket}, known as the \textit{center of oscillation}, in agreement with \cite{Rathinam1998}.
\end{example}

\begin{example}[name=Planar Aerial Manipulator, continues=example:planar_aerial_manipulator]
$\Delta$ can be verified to be orthogonal to the image of the global section
\begin{equation}
	\begin{gathered}
		\sigma \,  :\,  \mathbb{T}^1 \rightarrow SE(2) \times \mathbb{T}^1 \\ 
		\renewcommand*{\arraystretch}{1.3}
		\phi\mapsto 
		\left(\tfrac{\ell_q \, m_q}{m_g + m_q} \cos\phi, \, \tfrac{\ell_q \, m_q}{m_g + m_q} \sin\phi, \,0, \,\phi\right)
		\end{gathered}	
\end{equation}
since its differential takes values in the orthogonal complement of $\Delta$, which can be computed as
\begin{equation}
	\begin{gathered}
		\setlength\arraycolsep{1pt}
		\Delta^\perp = 
		\begin{bmatrix} 
			-c_1 \sin \theta - 2 \, c_2 \sin(\phi + \theta) & -c_2 \sin(\phi + \theta) \\
			\hphantom{-}c_1\cos\theta + 2 \, c_2\cos(\phi + \theta) &  \hphantom{-} c_2 \cos(\phi + \theta) \\
			2 \, c_3 &                   0 \\
			0 &             c_3
		\end{bmatrix} 
		\\
		c_1 = \ell_g \, (m_g + 2 \, m_q),  \ c_2 = \ell_q \,  m_q, \ c_3 = m_g + m_q.
	\end{gathered}
\end{equation}
The vertical and horizontal subbundles are given by
\begin{equation}
	\setlength\arraycolsep{2pt}
		VQ = \begin{bmatrix}
		1 & 0 & -x_2 \\
		0 & 1 &  x_1 \\
		0 & 0 &  1 \\
		0 & 0 &  0
	\end{bmatrix}, \ 
	HQ = 
	\begin{bmatrix}  
		-\sin\left({\phi}+{\theta}\right)\\ 
		\hphantom{-} \cos\left({\phi}+{\theta}\right)\\ 
		0\\ 
		\frac{m_g + m_q}{\ell_q \, m_q} \end{bmatrix}
\end{equation}
with which the implicit dynamics can be shown to be regular. Thus, the group variables of the canonical trivialization are a global \flatmodifier flat output:
\begin{equation}
	\begin{gathered}
		y :
		SE(2) \times \mathbb{T}^1 \rightarrow SE(2) \\ 
		\begin{aligned}
		q \mapsto
		\big(
			x_1 + c_4 \cos(\phi + \theta), 
			x_2 + c_4 \sin(\phi + \theta),
			\theta 
		\big)
		\end{aligned}
	\end{gathered}
\end{equation}
where $c_4 = -\tfrac{\ell_q \, m_q}{m_g + m_q}$. In harmony with \cite{Welde2021}, this is the pose of the end effector if ${\ell_q = 0}$; otherwise, it is the pose of a frame parallel to the end effector frame, translated by an offset due to the eccentricity of the vehicle center of mass.

\end{example}

	\begin{example}[name=Quadrotor,continues=example:quadrotor]
	The principal bundle is nontrivial, so global sections do not exist. However, removing even a single fiber $\pi^{-1}(s_0)$ for any ${s_0 \in \mathbb{S}^2}$ yields a trivial bundle.	
	We thus define the almost global sections
	\begin{equation}
			\begin{gathered}
			\sigma_N : 
\mathbb{S}^2 \setminus \left\{
\, -{e}_3 	\,	\right\}
			\rightarrow SE(3) \\ 
			\begin{pmatrix}
				s_1 \\
				s_2 \\
				s_3
			\end{pmatrix}
			\mapsto
			\def\arraystretch{1}%
			\begin{pmatrix}
				1 - \frac{{s_1}^2}{s_3 + 1} & \frac{-s_1 s_2}{s_3 + 1}   & s_1 & 0 \\
				\frac{-s_1 s_2}{s_3 + 1}   & 1 - \frac{{s_2}^2}{s_3 + 1}  & s_2 & 0 \\
				-s_1            &      -s_2            &      s_3 & 0 \\
				0 & 0 & 0 & 1
			\end{pmatrix}
		\end{gathered},
	\end{equation}
	\begin{equation}
			\begin{gathered}
			\sigma_S : 
			\mathbb{S}^2 \setminus \left\{
			\, {e}_3 	\,	\right\}
			\rightarrow SE(3) \\ 
			\begin{pmatrix}
				s_1 \\
				s_2 \\
				s_3
			\end{pmatrix}
			\mapsto
			\def\arraystretch{1}%
			\begin{pmatrix}
				1 + \frac{{s_1}^2}{s_3 - 1} &  \frac{-s_1 s_2}{s_3 - 1} & s_1 & 0 \\
				\frac{s_1s_2}{s_3 - 1} &  - 1 - \frac{{s_2}^2}{s_3 - 1} & s_2 & 0 \\
				s_1 &                -s_2 & s_3 & 0 \\
				0 & 0 & 0 & 1
			\end{pmatrix}
		\end{gathered}.
	\end{equation}
$\Delta$ is orthogonal to the image of each section, since $\Delta$ spans the linear velocities, while only the rotational degrees of freedom vary along the sections, and the inertia matrix is block diagonal. The regularity condition can also be shown to hold, and thus the group variables of the canonical local trivialization for each local section are almost global \flatmodifier flat outputs, corresponding to the center of mass position and a so-called ``body-fixed yaw angle'' around the thrust axis, equivalent to the flat outputs proposed in \cite{Watterson2017}.

\end{example}

\section{Discussion}

The existence of flat outputs that are the group variables of a particular trivialization echoes the conclusion of \cite{Hatton2011}, namely that although the dynamics of a mechanical system with symmetry can be represented equivalently in any trivialization, certain trivializations can be particularly favorable for analysis and control. Furthermore, viewing flat outputs through the geometric lens of trivializations suggests at least a partial answer to the open question as to why flat outputs often consist of merely ``a set of points and angles'' \cite{Murray1995} as opposed to arbitrary functions. 

Because our approach considers flat outputs taking values in an arbitrary Lie group, instead of limiting ourselves to $\mathbb{R}^n$, we are able to obtain global flat outputs if a global section is used. This is highly advantageous for agile systems like aerial robots, which stray far from a nominal operating point on the configuration manifold.
When the bundle is nontrivial, global sections do not exist, providing an upper bound on the domain of \flatmodifier flat outputs. However, as Example \ref{quadrotor_example} indicates, the approach provides a principled means of generating a global atlas of overlapping local flat outputs generated from local sections, yielding a \textit{differentially flat hybrid system} \cite{Sreenath2013}
and enabling planning and control over the entire configuration manifold \cite{Watterson2017}. 

The most urgent direction for future work is to develop a systematic method for identifying such orthogonal sections. We suspect that representing a section using basis functions, similarly to \cite{Hatton2011}, could enable automatic identification of \flatmodifier flat outputs using numerical optimization. This would permit the application of flatness-based planning and control techniques to complex multibody systems, for which symbolic analysis is tedious or impractical.
We also hope to close the gap between our necessary condition and our sufficient condition, which we believe to be occupied by systems whose state and inputs depend on flat output derivatives of higher order, such as the last two entries of Table \ref{catalog_of_\flatmodifier_flat_systems}. 
Perhaps a recursive approach, in which the shape space is regarded as yet another bundle, could encompass those systems as well. Finally, the evolution of nonholonomic systems can also be described using an affine connection \cite{Lewis2000}, suggesting the possibility of extension to systems with velocity constraints.

\section{Conclusion}

In this work, we formally define and explore the concept of \flatmodifier flat outputs for robotic systems evolving on principal bundles.  Under mild regularity assumptions, we use the symmetry of the system to construct a flat output from any section of the system's principal bundle that is orthogonal to an easily-computed distribution. These configuration flat outputs are equivariant and often global or almost global.
Similar to classic results in locomotion on principal bundles, a principal connection plays a key role in the analysis; however, our connection is flat, whereas other locomotory phenomena emerge specifically due to curvature. The results offer new fundamental insights into the dynamics of the broad class of mechanical systems without external constraints, including such free-flying systems as aerial and space robots. Most importantly, our approach enables the application of flatness-based planning and control 
approaches to new robotic systems
by facilitating the discovery of flat outputs with strong, useful properties.

\balance 
\bibliographystyle{IEEEtran}
\bibliography{IEEEabrv,cleaned}

\end{document}